\documentclass[10pt, a4paper]{article}

\usepackage[a4paper]{geometry}

\usepackage{url}
\usepackage{listings}
\usepackage{amssymb}
\usepackage{graphicx}
\usepackage{algcompatible}
\usepackage{algorithm}
\usepackage{url}
\usepackage{rotating}
\usepackage{booktabs}
\usepackage{subfig}
\usepackage{amsmath}
\usepackage{bbm}

\renewcommand{\labelenumi}{(\alph{enumi})}
\renewcommand\theenumi\labelenumi

\usepackage[english]{babel}

\usepackage[utf8]{inputenc}
\usepackage{xspace}
\usepackage{amsmath,amsthm,amssymb,mathtools}
\usepackage{lmodern}

\usepackage[algo2e,ruled,vlined,linesnumbered]{algorithm2e}

\usepackage{xcolor}
\usepackage{tikz}
\usepackage{graphicx}

\allowdisplaybreaks[4]
\newtheorem{theorem}{Theorem}
\newtheorem{lemma}[theorem]{Lemma}
\newtheorem{corollary}[theorem]{Corollary}

\newcommand{\om}{\textsc{OneMax}\xspace}
\newcommand{\onemax}{\om}
\newcommand{\lo}{\textsc{LeadingOnes}\xspace}
\newcommand{\leadingones}{\lo}

\newcommand{\R}{\ensuremath{\mathbb{R}}}

\newcommand{\N}{\ensuremath{\mathbb{N}}} 

\DeclareMathOperator{\Bin}{Bin}

\date{}
\newcommand{\Var}{\mathrm{Var}\xspace} 

\begin{document}
%
%
%
%

\title{Sharp Bounds for Genetic Drift in Estimation of Distribution Algorithms\thanks{A small subset of the results presented in this work were already stated, without proof or proof idea, in the conference paper~\cite[Theorem~4.5]{ZhengYD18}, namely that the expected time the frequency of a neutral bit takes to hit the absorbing states $0$ or $1$ is $\Theta(K^2)$ for cGA and $\Theta(\mu)$ for UMDA. Our work now extends the UMDA result to the PBIL, strenghthens all lower bounds by regarding the event of leaving the middle range $[\frac 14,\frac 34]$ of the frequency range, adds a tail bound for the lower bounds, and adds domination arguments allowing to extend the lower bounds to bits that are neutral or prefer a particular value. Also, complete proofs are given for all results. Both authors contributed equally to this work and both act as corresponding authors. }}

\author{Benjamin Doerr\\ Laboratoire d'Informatique (LIX)\\ CNRS\\ \'Ecole Polytechnique\\ Institute Polytechnique de Paris\\ Palaiseau, France
\and Weijie Zheng\\ Shenzhen Key Laboratory of Computational Intelligence\\ University Key Laboratory of Evolving Intelligent Systems of Guangdong Province\\ Department of Computer Science and Engineering \\ Southern University of Science and Technology \\ Shenzhen, China}

\maketitle

\begin{abstract}
Estimation of Distribution Algorithms (EDAs) are one branch of Evolutionary Algorithms (EAs) in the broad sense that they evolve a probabilistic model instead of a population. Many existing algorithms fall into this category. Analogous to genetic drift in EAs, EDAs also encounter the phenomenon that updates of the probabilistic model not justified by the fitness move the sampling frequencies to the boundary values. This can result in a considerable performance loss. 

This paper proves the first sharp estimates of the boundary hitting time of the sampling frequency of a neutral bit for several univariate EDAs. For the UMDA that selects $\mu$ best individuals from $\lambda$ offspring each generation, we prove that the expected first iteration when the frequency of the neutral bit leaves the middle range $[\tfrac 14, \tfrac 34]$ and the expected first time it is absorbed in 0 or 1 are both $\Theta(\mu)$. The corresponding hitting times are $\Theta(K^2)$ for the cGA with hypothetical population size $K$. This paper further proves that for PBIL with parameters $\mu$, $\lambda$, and $\rho$, in an expected number of $\Theta(\mu/\rho^2)$ iterations the sampling frequency of a neutral bit leaves the interval $[\Theta(\rho/\mu),1-\Theta(\rho/\mu)]$ and then always the same value is sampled for this bit, that is, the frequency approaches the corresponding boundary value with maximum speed.

For the lower bounds implicit in these statements, we also show exponential tail bounds. If a bit is not neutral, but neutral or has a preference for ones, then the lower bounds on the times to reach a low frequency value still hold. An analogous statement holds for bits that are neutral or prefer the value zero. 
%
\end{abstract}

\section{Introduction}
Estimation of Distribution Algorithms (EDAs) are evolutionary algorithms (EAs) that evolve a probabilistic model instead of a population. An iteration of an EDA usually consists of three steps. (i)~Based on the current probabilistic model, a population of individuals is sampled. (ii)~The fitness of this population is determined. (iii)~Update of the probabilistic model: Based on the fitness of this population and the probabilistic model, a new probabilistic model is computed.

Different probabilistic models and update strategies form different specific algorithms in this framework. In multivariate EDAs, the probabilistic model contains dependencies among the variables. Examples for multivariate EDAs include Mutual-Information-Maximization Input Clustering~\cite{BonetIV96}, Bivariate Marginal Distribution Algorithm~\cite{PelikanM99}, the Factorized Distribution Algorithm~\cite{MuhlenbeinM99}, the Extended Compact Genetic Algorithm~\cite{HarikLS06}, and many others. 

For univariate EDAs, the bit positions of the probabilistic model are mutually independent. Univariate EDAs include Population-Based Incremental Learning (PBIL)~\cite{Baluja94,BalujaC95} 
with special cases Univariate Marginal Distribution Algorithm (UMDA)~\cite{MuhlenbeinP96} and Max-Min Ant System with iteration-best update (MMAS$_{ib}$)~\cite{NeumannSW10}, and the Compact Genetic Algorithm (cGA)~\cite{HarikLG98}. Since the dependencies in multivariate EDAs bear significant difficulties for a mathematical analysis, almost all theoretical results for EDAs regard univariate models~\cite{KrejcaW20bookchapter}. This paper also deals exclusively with univariate EDAs. 

In evolutionary algorithms, it is known that the frequencies of bit values in the population are not only influenced by the contribution of the bit to the fitness, but also by random fluctuation stemming from other bits having a stronger influence on the fitness. These random fluctuations can even lead to certain bits converging to a single value different from the one in the optimal solution. This effect is called \emph{genetic drift}~\cite{Kimura64,AsohM94}. 

Genetic drift also happens in EDAs. Gonz{\'a}lez, Lozano, and Larra{\~n}aga~\cite{GonzalezLL01} showed that for the $2$-dimensional \onemax function, the sampling frequency of PBIL can converge to any search point in the search space with probability near to $1$ if the initial sampling frequency goes to that search point and the learning rate goes to~1. Droste~\cite{Droste06} noticed the possibility of the cGA getting stuck, but he only analyzed the runtime conditional on it being finite, and no analysis of genetic drift or stagnation times was given. Costa, Jones, and Kroese~\cite{CostaJK07} proved that a constant smoothing parameter for the Cross Entropy (CE) algorithm (which is equivalent to a constant learning rate $\rho$ for PBIL) results in that the probability mass function converges to a unit mass at some random candidate, but no convergence speed analysis was given. In summary, as Krejca and Witt said in~\cite{KrejcaW20bookchapter}, the genetic drift in EDAs is a general problem of martingales, that is, that a random process with zero expected change will eventually stop at the absorbing boundaries of the range. Witt~\cite{Witt19} and Lengler, Sudholt, and Witt~\cite{LenglerSW18} recently showed that genetic drift can result in a considerable performance loss on the \onemax function.

In this work, we shall quantify this effect asymptotically precisely for several EDAs and this via proven results. The few \textbf{previous works} in this direction have obtained the following results. Friedrich, K\"otzing, and Krejca~\cite{FriedrichKK16} showed that for the cGA with hypothetical population size $K$, the frequency of a neutral bit position is arbitrary close to the borders $0$ or $1$ after expected $\omega(K^2)$ generations. Though not stated in~\cite{FriedrichKK16}, from its Corollary~9, we can derive an upper bound of $O(K^2)$ for the expected time of leaving the interval $[\frac 14, \frac 34]$, and $O(K^2 \log K)$ for expected hitting time of a boundary value. For the UMDA selecting $\mu$ best individuals from $\lambda$ offspring, the situation is similar~\cite{FriedrichKK16}. After $\omega(\mu)$ iterations, the frequencies are arbitrary close to the boundaries and the expected hitting time can be shown to be $O(\mu \log \mu)$ via similar arguments as above. Sudholt and Witt~\cite{SudholtW19} mentioned that the boundary hitting time of the cGA is $\Theta(K^2)$, but without a complete proof (in particular, because they did not discuss what happens once the frequency leaves the interval $[\frac 16, \frac 56]$). Although Krejca and Witt~\cite{KrejcaW20} focused on the lower bound of the runtime of the UMDA on \onemax, we can derive from it that the hitting time of the boundary $0$ is at least $\Omega(\mu)$. This follows from the drift of $\phi$ in Lemma~9 in~\cite{KrejcaW20} together with the additive drift theorem~\cite{HeY01}. 

\textbf{Our results:} While the results above give some indication on the degree of stability of PBIL and the cGA, a sharp proven result is still missing. This paper overcomes this shortage and gives precise asymptotic hitting times for PBIL (including the UMDA and the MMAS$_{ib}$) and the cGA. With a simultaneous analysis of the UMDA and the cGA, we prove that for the UMDA selecting $\mu$ best individuals from $\lambda$ offspring on some $D$-dimensional problem, the expected number of iterations until the frequency of the neutral bit position is absorbed in 0 or 1 for the UMDA without margins or when the frequency hits the margins $\{1/D,1-1/D\}$ for the UMDA with such margins is $\Theta(\mu)$, and the corresponding hitting time is $\Theta(K^2)$ for the cGA with hypothetical population size $K$. This paper also gives a precise asymptotic analysis for PBIL selecting $\mu$ best individuals from $\lambda$ offspring and with a learning rate of $\rho$: In expectation in $\Theta(\mu/\rho^2)$ generations the sampling frequency of a neutral bit position leaves the interval $[\Theta(\rho/\mu),1-\Theta(\rho/\mu)]$ and then always the same value is sampled for this position.

For the lower bounds implicit in these estimates we prove an exponential tail bound in Theorem~\ref{thm:tail}. 

We also extend the lower bound results to bit positions that are neutral or have a preference for some bit value (Section~\ref{sec:lowerP}). For example, we prove that for PBIL it takes an expected number of $\Omega(\mu/\rho^2)$ iterations until the sampling frequency of a position that is neutral or prefers a one (neutral or prefers a zero) reaches the interval $[0,\tfrac 14]$ ($[\tfrac 34, 1]$). The corresponding hitting time is $\Omega(K^2)$ for the cGA.

The remainder of this paper is organized as follows. Section~\ref{sec:IntroEDA} briefly introduces PBIL and the cGA under the umbrella of the $n$-Bernoulli-$\lambda$-EDA framework proposed in~\cite{FriedrichKK16}. Our notation for our results is fixed in Section~\ref{sec:Notation}. Section~\ref{sec:lower} and Section~\ref{sec:upper} discuss how fast the frequency of a neutral bit position approaches the boundaries. Section~\ref{sec:lowerP} extends the lower bound results of Section~\ref{sec:lower} to bit positions that are neutral or have some preference. Finally, in Section~\ref{sec:disc} we argue how our results allow to interpret existing research results and how they give hints on how to choose the parameters of these EDAs.

\section{The $n$-Bernoulli-$\lambda$-EDA Framework}
\label{sec:IntroEDA}

Since the $n$-Bernoulli-$\lambda$-EDA framework proposed in~\cite{FriedrichKK16} covers many well-known EDAs including PBIL and the cGA, we use it to make precise these two EDAs. 

We note that often \emph{margins} like $1/D$ and $1-1/D$ are used, that is, the frequencies are restricted to stay in the interval $[1/D,1-1/D]$. This prevents the frequencies from reaching the absorbing states $0$ and $1$. To ease the presentation, we regard the EDAs without such margins. We note that, trivially, the time to reach an absorbing state is not smaller than the time to reach a margin value. Hence an upper bound on the hitting time of the absorbing states is also an upper bound for the time to reach or exceed the margin values. Our main result on lower bounds, Corollary~\ref{cor:lower}, shows a lower bound for the time to reach a frequency value in $[0,\frac14] \cup [\frac 34, 1]$. This again is a lower bound for the time to reach (or exceed) the margin values or the absorbing states. 

The $n$-Bernoulli-$\lambda$-EDA framework for maximizing a function $f: \{0,1\}^D \rightarrow \R$ is shown in Algorithm~\ref{alg:EDA}. By suitably specifying the update scheme $\phi$, we derive PBIL and the cGA. The general idea of \emph{population-based incremental learning (PBIL)} is to sample $\lambda$ individuals from the current distribution, select $\mu$ best of them, and use these (with a \emph{learning rate} of $\rho$) and the current distribution to define the new distribution. Formally, the update scheme is
\begin{equation}
\begin{split}
p_j^t={}&\varphi(p^{t-1}, (X_i, f(X_i))_{i=1,\dots,\lambda})_j \\
={}&(1-\rho)p_j^{t-1}+\frac{\rho}{\mu}\sum\limits_{i=1}^{\mu}\tilde{X}_{i,j}^t,
\label{eq:pbilupdate}
\end{split}
\end{equation}
where $\rho$ is the learning rate and $\tilde{X}_1^t,...,\tilde{X}_\mu^t$ are the selected $\mu$ best individuals from the $\lambda$ offspring. 

The \emph{cross entropy} algorithm (CE) has various definitions according to the problems to be solved. The basic CE algorithm for discrete optimization~\cite{CostaJK07} samples $N$ individuals from the current distribution, selects $N_b$ best of them, and uses these (with a time-dependent \emph{smoothing rate} of $\alpha_t$) and the current distribution to define the new distribution. The formal update scheme is (\ref{eq:pbilupdate}) with $\mu$, $\lambda$ and $\rho$ respectively replaced by $N_b$, $N$ and $\alpha_t$. The basic CE is equal to PBIL except that the learning rate is fixed for PBIL, whereas CE utilizes time-dependent learning rates. When referring to the CE algorithm in this paper, we mean this version from~\cite{CostaJK07}, but we denote its parameters by $\mu$, $\lambda$ and $\rho_t$ instead of $N_b$, $N$ and $\alpha_t$ to reflect the similarity with PBIL.

Two special cases of PBIL have been regarded in the literature. The \emph{univariate marginal distribution algorithm (UMDA)} only uses the samples of this current iteration to define the next probabilistic model, hence it is equivalent to PBIL with a learning rate of $\rho=1$. The \emph{$\lambda$-max-min ant system ($\lambda$-MMAS)} only selects the best sampled individual and the current model to construct the new model, hence it is the special case with $\mu=1$. 
\begin{algorithm}[!ht]
\caption{The $n$-Bernoulli-$\lambda$-EDA framework with update scheme $\varphi$ to maximize a function $f: \{0,1\}^D \rightarrow \R$}
{\small
 \begin{algorithmic}[1]
 \STATE{$p^0=(\tfrac{1}{2}, \tfrac{1}{2},\dots,\tfrac{1}{2})\in [0,1]^D$}
 \FOR {$t=1,2,\dots$}
 \FOR {$i=1,2,\dots,\lambda$}
 \STATEx {$\quad\%\%$\textsl{Sampling of the $i$-th individual $X_i^t=(X_{i,1}^t,\dots,X_{i,D}^t)$}}
 \FOR {$j=1,2,\dots,D$}
 \STATE $X_{i,j}^t \leftarrow 1$ with probability $p_{i}^{t-1}$ and
 \STATEx \quad\quad\quad$X_{i,j}^t \leftarrow 0$ with probability $1-p_{i}^{t-1}$; 
 \ENDFOR
 \ENDFOR
 \STATEx {$\quad\%\%$\textsl{Update of the frequency vector}}
 \STATE $p^t\leftarrow\varphi(p^{t-1}, (X_i, f(X_i))_{i=1,\dots,\lambda})$;
 \ENDFOR
 \end{algorithmic}
 \label{alg:EDA}
}
\end{algorithm}

The \emph{compact genetic algorithm (cGA)} with hypothetical population size $K$, not necessarily an integer, samples two individuals and then changes the frequency of each bit position by an absolute value of $1/K$ towards the bit value of the better individual (unless the two sampled individuals have identical values in this position). Formally, we have $\lambda=2$ in the $n$-Bernoulli-$\lambda$-EDA framework and the update scheme is
\begin{equation}
    \begin{split}
    p_j^t={}&\varphi(p^{t-1}, (X_i, f(X_i))_{i=1,\dots,\lambda})_j\\
    ={}& \begin{cases}
    p_j^{t-1}+\tfrac{1}{K}, & \text{if $X_{(1),j}^t>X_{(2),j}^t$}\\
    p_j^{t-1}-\tfrac{1}{K}, & \text{if $X_{(1),j}^t<X_{(2),j}^t$}\\
    p_j^{t-1}, & \text{if $X_{(1),j}^t=X_{(2),j}^t$},\\
    \end{cases}
    \end{split}
    \label{eq:cgaupdate}
\end{equation}
where $\{X_{(1)}^t,X_{(2)}^t\} = \{X_1^t,X_2^t\}$ such that $f(X_{(1)}^t) \ge f(X_{(2)}^t)$. We shall always make the following \emph{well-behaved frequency assumption} (first called so in~\cite{Doerr19gecco}, but made in many earlier works already): Any two frequencies the cGA can reach differ by a multiple of $1/K$. In the case of no margins, this means that the cGA can only use frequencies in $\{0, 1/K, 2/K, \dots, 1\}$. Note that $K$ needs to be even so that the initial frequency $1/2$ is also a multiple of $1/K$. When using the margins $1/D$ and $1- 1/D$, the set of reachable frequency boundaries is $\{1/D, 1/D+1/K, 1/D+2/K, \dots, 1 - 1/D\}$. To have $1/2$ in this set, $1 - 2/D$ needs to be an even multiple of $1/K$.

\section{Notation Used in Our Analyses}
\label{sec:Notation}
Genetic drift is usually studied via the behavior of a neutral bit position. Let $f: \{0,1\}^D \rightarrow \R$ be an arbitrary fitness function with a neutral bit position. Without loss of generality, let the first bit position of the fitness function $f$ be neutral, that is, we have $f(0,X_2,\dots,X_D)=f(1,X_2,\dots,X_D)$ for all $X_2,\dots,X_D\in\{0,1\}$. Then we can simply assume that $\tilde{X}_{i,1}^t=X_{i,1}^t,i=1,\dots,\mu$ in (\ref{eq:pbilupdate}), and $X_{(1),1}^t=X_{1,1}^t, X_{(2),1}^t=X_{2,1}^t$ in (\ref{eq:cgaupdate}).
Let $p_t=p_1^t$ be the frequency of the neutral bit position after generation $t$. For PBIL, we have
\begin{equation}
\begin{split}
p_t=
\begin{cases}
\frac{1}{2}, &t=0,\\
(1-\rho)p_{t-1}+\frac{\rho}{\mu}\sum\limits_{i=1}^{\mu}X_{i,1}^t, &t\geq 1,
\end{cases}
\end{split}
\label{eq:pbil}
\end{equation}
where the $X_{i,1}^t$ are independent $0,1$ random variables with $\Pr[X_{i,1}^t=1]=p_{t-1}$. 

For the cGA, we have
\begin{equation*}
\begin{split}
p_t=
\begin{cases}
\frac{1}{2}, &t=0,\\
 \begin{cases}
    p_{t-1}+\frac{1}{K}, & \text{if $X_{1,1}^t>X_{2,1}^t$}\\
    p_{t-1}-\frac{1}{K}, & \text{if $X_{1,1}^t<X_{2,1}^t$}\\
    p_{t-1}, & \text{if $X_{1,1}^t=X_{2,1}^t$}\\
    \end{cases}
, &t\geq 1,
\end{cases}
\end{split}
\end{equation*}
where $X_{1,1}^t$ and $X_{2,1}^t$ are independent $0,1$ random variables with $ \Pr[X_{1,1}^t=1]= \Pr[X_{2,1}^t=1]=p_{t-1}$. 

We observe that this random process $(p_t)$ is independent of $f,D$, and, in the case of PBIL, $\lambda$. We also have
\begin{equation*}
E[p_t\mid p_{t-1}]=p_{t-1},
\end{equation*}
that is, both PBIL and the cGA are balanced in the sense of~\cite{FriedrichKK16}. 

Finally, let  $T=\min\{t\mid p_t \in \{0,1\}\}$ be the hitting time of the absorbing states 0 and 1.

We are now ready to prove our matching upper and lower bounds for the hitting time $T$. We start with lower bounds in Section~\ref{sec:lower} as these are easier to prove and thus a good warm-up for the upper bound proofs in Section~\ref{sec:upper}.

\section{Lower Bounds on the Boundary Hitting Time}
\label{sec:lower}

To prove our lower bounds, we use the following version of the Hoeffding-Azuma inequality for maxima, see ~\cite[Theorem 3.10 and (41)]{McDiarmid98} and note that in (41) the absolute value should be inside the maximum, as can be seen from the proof in~\cite{McDiarmid98}.
\begin{theorem}[\cite{McDiarmid98}]
Let $a_1,\dots,a_m \in \R$, and $S_1,\dots,S_m$ be a martingale difference sequence with $|S_k| \le a_k$ for each $k$. Then for any $s \ge 0$,
\begin{align*}
\Pr\left[\max\limits_{k=1,\dots,m} \left|\sum_{i=1}^k S_i\right| \ge s\right]\le 2\exp\left(-\frac{s^2}{2\sum_{i=1}^m a_i^2}\right).
\end{align*}
\label{thm:HoeffdingAzuma}
\end{theorem}

Now we prove our lower bounds. We first derive tail bounds in Theorem~\ref{thm:tail}, and use the tail bounds to further obtain our lower bounds on the expected hitting time of the absorbing states. The expectations of hitting times are asymptotically equal to (and necessarily not less than) the expected times of leaving the frequency range $(\frac 14, \frac 34)$, so we determine these in Corollary~\ref{cor:lower}, which are also of independent interest.

\begin{theorem}\label{thm:tail}
  Consider using an $n$-Bernoulli-$\lambda$-EDA to optimize some function $f$ with a neutral bit position. Let $p_t, t = 0, 1, 2, \dots$ denote the frequency of the neutral bit position after iteration $t$. 
  \begin{enumerate}
  \item If the EDA is PBIL with learning rate $\rho$ and selection size $\mu$, then for all $\gamma > 0$ and $T \in \N$ we have
  \[\Pr[\forall t \in [0..T] : |p_t - \tfrac 12| < \gamma] \ge 1 - 2 \exp\left({-\frac{\gamma^2 \mu}{2\rho^2 T}}\right).\]
  \item If the EDA is the cGA with hypothetical population size $K$, then for all $\gamma > 0$ and $T \in \N$ we have
  \[\Pr[\forall t \in [0..T] : |p_t - \tfrac 12| < \gamma] \ge 1 - 2 \exp\left({-\frac{\gamma^2 K^2}{2 T}}\right).\]
  \end{enumerate} 
\end{theorem}

\begin{proof}
For PBIL, building on the notation introduced in Section~\ref{sec:Notation}, we consider the random process
\begin{equation*}
Z_{t\mu+a}=(1-\rho)p_t\mu+\rho p_{t}(\mu-a)+\rho\sum\limits_{i=1}^aX_{i,1}^{t+1},
\end{equation*}
where $t=0,1,\dots$, and $a=0,1,\dots,\mu-1$. For $a=0$, we obviously have $Z_{t\mu}/{\mu}=p_{t}$, that is, the $Z$-process contains the process $(p_t)$ we are interested in. 
Noting that $Z_{(t+1)\mu}$ can also be written as $Z_{t\mu+\mu}=(1-\rho)p_t\mu+\rho p_{t}(\mu-\mu)+\rho\sum\limits_{i=1}^{\mu}X_{i,1}^{t+1}$, it is also not difficult to see that for all $k=0,1,\dots$, we have
\begin{equation}
\begin{split}
\Pr[Z_{k+1}={}&Z_k+\rho-\rho p_t\mid Z_{1},\dots,Z_{k}]=p_t,\\
\Pr[Z_{k+1}={}&Z_k+0-\rho p_t\mid Z_1,\dots,Z_k]=1-p_t,
\end{split}
\label{eq:nprob}
\end{equation}
where $t = \lfloor k/\mu \rfloor$. 
Consequently, 
\begin{equation*}
E[Z_{k+1}\mid Z_1,\dots,Z_k]=Z_k
\end{equation*}
and the sequence $Z_0,Z_1,Z_2,\dots$ is a martingale. For $k=1,2,\dots$, let $R_k=Z_k-Z_{k-1}$ define the martingale difference sequence. By (\ref{eq:nprob}),
\begin{equation*}
|R_k|\le \max\{\rho(1-p_t),\rho p_t\} \le \rho.
\end{equation*}
By the Hoeffding-Azuma inequality (Theorem~\ref{thm:HoeffdingAzuma}), we have
\begin{equation}
\Pr \left[ \max\limits_{k=1,\dots,t\mu}\left|\sum\limits_{i=1}^k{R_i} \right|\ge M \right]\le 2\exp \left({-\frac{M^2}{2t\mu\rho^2}} \right).
\label{eq:azumaMM}
\end{equation}
Recalling $Z_0 = \frac {\mu}2$ and $p_t = Z_{t\mu}/\mu$, we have 
\begin{equation}
\begin{split}
\Pr {}&\left[ \max \limits_{k=1,\dots,t} \left| p_k - \tfrac 12  \right| \ge M/\mu \right]\\
\le {}&
\Pr \left[ \max\limits_{k=1,\dots,t\mu} \left|\sum\limits_{i=1}^k{R_i} \right|\ge M \right].\\
\end{split}
\label{eq:pZ}
\end{equation}
Combining \eqref{eq:azumaMM} and \eqref{eq:pZ} with $M=\gamma \mu$, we obtain
\[\Pr \left[ \max\limits_{k=1,\dots,t} \left| p_k - \tfrac 12  \right| \ge \gamma \right]
\le 2 \exp\left({-\frac{\gamma^2}{2t\rho^2}}\right)\]
and thus we prove the result for PBIL.

For the cGA, we may simply regard the process $Z_k=p_k$. Since for all $k=0,1,\dots$, 
\begin{align*}
\Pr[Z_{k+1}={}&Z_k+\tfrac{1}{K}\mid Z_{1},\dots,Z_{k}]=p_k(1-p_k),\\
\Pr[Z_{k+1}={}&Z_k-\tfrac{1}{K}\mid Z_1,\dots,Z_k]=p_k(1-p_k),\\
\Pr[Z_{k+1}={}&Z_k\mid Z_1,\dots,Z_k]=1-2p_k(1-p_k),
\end{align*}
we have $E [Z_{k+1}\mid Z_1,\dots,Z_k ]=Z_k$. The martingale difference sequence $R_k:=Z_k-Z_{k-1}$ satisfies $|R_k|\le\tfrac{1}{K}$. By the Hoeffding-Azuma inequality, we have
\begin{align*}
\Pr {}&\left[ \max\limits_{k=1,\dots,t} \left|p_k - \tfrac 12\right | \ge M \right] \\
={}& \Pr \left[ \max\limits_{k=1,\dots,t} \left|\sum\limits_{i=1}^k{R_i} \right | \ge M \right]
\le 2 \exp \left({-\frac{M^2K^2}{2t}} \right).
\end{align*}
Taking $M=\gamma$ we prove our result for the cGA.
\end{proof}

Let $T_0$ the first time the frequency of the neutral bit position is in $[0,\frac 14] \cup [\frac 34, 1]$. Then we know $T_0=\min \{t \mid |p_t-\tfrac{1}{2} | \ge \tfrac{1}{4} \}$. Hence, via taking $T=\mu/(32\rho^2)$ for PBIL and $T=K^2/32$ for the cGA in Theorem~\ref{thm:tail}, we could easily obtain the expected hitting time, as shown in Corollary~\ref{cor:lower}.
\begin{corollary}
Consider using an $n$-Bernoulli-$\lambda$-EDA to optimize some function $f$ with a neutral bit position. Let $T_0$ denote the first time the frequency of the neutral bit position is in $[0,\frac 14] \cup [\frac 34, 1]$. For PBIL, we have $E[T_0]=\Omega(\tfrac{\mu}{\rho^2})$, in particular, $E [T_0]=\Omega({\mu})$ for the UMDA and $E [T_0 ]=\Omega(\tfrac{1}{\rho^2})$ for the $\lambda$-MMAS. For the cGA, we have $E[T_0]=\Omega(K^2)$. 
\label{cor:lower}
\end{corollary}

We note that the lower bound proof for PBIL can be extended to CE, either by simply replacing $\rho$ by the  supremum $\rho_{\sup} = \sup\{\rho_t \mid t \in \N\}$ and obtaining a lower bound of $\Omega(\mu/\rho_{\sup}^2)$, or by replacing $t\rho^2$ in~\eqref{eq:azumaMM} by $\sum_{s=1}^t \rho_t^2$. With a suitable choice of $t$, this gives a bound taking into account the particular values of $(\rho_t)$. We omit the details.



\section{Upper Bounds on the Boundary Hitting Time}
\label{sec:upper}

We now prove that, roughly speaking, the lower bounds shown in the previous section are asymptotically tight. 
To prove our upper bounds, we use the following two lemmas.
\begin{lemma}
For all $z\ge 0$ and $z_0>0$, we have
\begin{equation*}
\begin{split}
\sqrt z\leq {}& \sqrt z_0+\tfrac{1}{2}z_0^{-1/2}(z-z_0)-\tfrac{1}{8}z_0^{-3/2}(z-z_0)^2\\
{}&+\tfrac{1}{16}z_0^{-5/2}(z-z_0)^3.
\end{split}
\end{equation*}
\label{lem:z}
\end{lemma}
\begin{proof}
For the convenience of the proof, let $x=\sqrt z$ and $a=\sqrt {z_0}$. We consider the function
\begin{equation*}
\begin{split}
g(x)={}&x-a-\tfrac{1}{2}a^{-1}(x^2-a^2)+\tfrac{1}{8}a^{-3}(x^2-a^2)^2\\
{}&-\tfrac{1}{16}a^{-5}(x^2-a^2)^3\\
={}&{-\tfrac{1}{16}}a^{-5}x^6+\tfrac{5}{16}a^{-3}x^4-\tfrac{15}{16}a^{-1}x^2+x-\tfrac{5}{16}a
\end{split}
\end{equation*}
and show that $g(x) \le 0$. Since
\begin{equation*}
g'(x)={-\tfrac{3}{8}}a^{-5}x^5+\tfrac{5}{4}a^{-3}x^3-\tfrac{15}{8}a^{-1}x+1
\end{equation*}
and
\begin{equation*}
\begin{split}
g''(x)={}&{-\tfrac{15}{8}}a^{-5}x^4+\tfrac{15}{4}a^{-3}x^2-\tfrac{15}{8}a^{-1}\\
={}&{-\tfrac{15}{8}}a^{-5}(x^4-2a^2x^2+a^4)\\
={}&{-\tfrac{15}{8}}a^{-5}(x^2-a^2)^2\le 0,
\end{split}
\end{equation*}
we know that $g'(x)$ is monotonically decreasing. Since $g'(0)=1$ and $g'(a)=0$, we observe that $g(x)$ increases in $[0,a)$ and decreases in $[a,\infty)$. Therefore, ${g(x)\le g(a)=0}$.
\end{proof}

An easy calculation gives the following second-order and third-order central moments of the frequency of the neutral bit position in PBIL and the cGA.
\begin{lemma}
For PBIL, we have
\begin{equation*}
\begin{split}
&\Var[p_t\mid p_{t-1}]=\frac{\rho^2}{\mu}p_{t-1}(1-p_{t-1}),\\
&E[(p_t-E[p_t\mid p_{t-1}])^3\mid p_{t-1}] \\
&\quad = \frac{\rho^3}{\mu^2}p_{t-1}(1-p_{t-1})(1-2p_{t-1}).
\end{split}
\end{equation*}
For the cGA, we have
\begin{equation*}
\begin{split}
&\Var[p_t\mid p_{t-1}]=\frac{2}{K^2}p_{t-1}(1-p_{t-1}),\\
&E[(p_t-E[p_t\mid p_{t-1}])^3\mid p_{t-1}]=0.
\end{split}
\end{equation*}
\label{lem:moments}
\end{lemma}
\begin{proof}
For PBIL, note that $\sum_{i=1}^{\mu}X_{i,1}^t \sim \Bin(\mu, p_{t-1})$. Thus we have
\begin{align*}
\Var{}&\left[ \sum_{i=1}^{\mu}X_{i,1}^t \,\middle|\, p_{t-1} \right] = \mu p_{t-1}(1-p_{t-1}),\\
E{}& \left[ \left(\sum_{i=1}^{\mu}X_{i,1}^t-E\left[\sum_{i=1}^{\mu}X_{i,1}^t \,\middle|\, p_{t-1} \right]\right)^3 \,\middle|\, p_{t-1}\right] \\
&={} \mu p_{t-1}(1-p_{t-1})(1-2p_{t-1}).
\end{align*}
Hence, recalling that $p_t=(1-\rho)p_{t-1}+\frac{\rho}{\mu}\sum\limits_{i=1}^{\mu}X_{i,1}^t$, and noting that centered moments are invariant to translations with constants and that constant scaling factors can be pulled out in the corresponding power, we have
\begin{align*}
\Var{}&[p_t\mid p_{t-1}]=\left(\frac{\rho}{\mu}\right)^2\Var\left[ \sum_{i=1}^{\mu} X_{i,1}^t\,\middle|\, p_{t-1}\right]\\
&={}\frac{\rho^2}{\mu}p_{t-1}(1-p_{t-1})
\end{align*}
and
\begin{align*}
E{}&[(p_t-E[p_t \mid p_{t-1}])^3\mid p_{t-1}] \\
&={} \left(\frac{\rho}{\mu}\right)^3 E\left[ \left(\sum_{i=1}^{\mu}X_{i,1}^t-E\left[\sum_{i=1}^{\mu}X_{i,1}^t \,\middle|\, p_{t-1} \right]\right)^3 \,\middle|\, p_{t-1}\right]\\
&={} \frac{\rho^3}{\mu^2}p_{t-1}(1-p_{t-1})(1-2p_{t-1}).
\end{align*}

For the cGA, we compute
\begin{align*}
\Var{}&[p_t\mid p_{t-1}]=E[(p_t-E[p_t\mid p_{t-1}])^2 \mid p_{t-1}]\\
={}&p_{t-1}(1-p_{t-1}) \left(\frac{1}{K}\right)^2+ p_{t-1}(1-p_{t-1})\left({-\frac{1}{K}}\right)^2\\
={}&\frac{2}{K^2}p_{t-1}(1-p_{t-1})
\end{align*}
and
\begin{align*}
E[({}&p_t-E[p_t\mid p_{t-1}])^3\mid p_{t-1}] \\
={}&p_{t-1}(1-p_{t-1}) \left(\frac{1}{K}\right)^3+ p_{t-1}(1-p_{t-1})\left({-\frac{1}{K}}\right)^3=0.
\qedhere
\end{align*}
\end{proof}

We are now ready to prove the following upper bounds for the hitting time of the absorbing states of the frequency of a neutral bit position. We consider EDAs without margins here, 
but since the time to reach an absorbing state is not smaller than the time to reach a margin value, we know that an upper bound on the hitting time of the absorbing states is also an upper bound for the time to hit a margin value when margins are used.

\begin{theorem}
Consider using an $n$-Bernoulli-$\lambda$-EDA to optimize some function $f$ with a neutral bit position. 
\begin{itemize}
  \item If the EDA is PBIL with $\rho < 1$, including the case of the $\lambda$-MMAS, then the following holds. Let $c\in (\tfrac{1}{2},\tfrac{1}{\sqrt 2})$. We say that the frequency $p_t$ of the neutral bit position \emph{runs away} from time $t$ onwards if 
  \begin{enumerate}
  \item $p_t \le c \frac{\rho}{\mu}$ and in all iterations $t' > t$ all samples have a zero in the neutral bit position, or
  \item $p_t \ge 1 - c \frac{\rho}{\mu}$ and in all iterations $t' > t$ all samples have a one in the neutral bit position.
  \end{enumerate}
  For $\tilde{T}$ denoting the first $t$ such that $p_t$ runs away from time $t$ on, we have $E[\tilde{T}] = O(\tfrac{\mu}{\rho^2})$.
  \item If the EDA is the UMDA, that is, PBIL with $\rho = 1$, then the first hitting time $T$ of the absorbing states $\{0, 1\}$ satisfies $E[T] = O(\mu)$.
  \item For the cGA, the expected first time to reach an absorbing state satisfies $E[T]=O(K^2)$.
  \end{itemize}
\label{thm:upper} 
\end{theorem}
\begin{proof}
Let $q_t=\min\{	p_t,1-p_t\}$ and $Y_t=\sqrt{q_t}$. Then $T=\min\{t\mid q_t=0\}$ and $\tilde{T}=\min\{t\mid q_t \le c\tfrac{\rho}{\mu}\}$. Due to the symmetry, we just discuss the case when $q_{t-1}=p_{t-1}$. Obviously, $p_{t-1}\le \tfrac{1}{2}$ in this case. Let us assume that $p_{t-1} > c\tfrac{\rho}{\mu}$. Using Lemma \ref{lem:z} with $z=p_t$ and $z_0=p_{t-1}$, we have
\begin{equation*}
\begin{split}
E[\sqrt{p_t}\mid p_{t-1}]
\le{}& E[Y_{t-1} |\, p_{t-1}]+\tfrac{1}{2}p_{t-1}^{-1/2}E[p_t-p_{t-1}  |\, p_{t-1}]\\
{}&-\tfrac{1}{8}p_{t-1}^{-3/2}E [(p_t-p_{t-1})^2\mid p_{t-1}] \\
{}&+ \tfrac{1}{16}p_{t-1}^{-5/2}E [(p_t-p_{t-1})^3\mid p_{t-1} ]
\end{split}
\end{equation*}
and thus
\begin{equation}
\begin{split}
E[Y_{t-1}-\sqrt{p_t}\mid {}&Y_{t-1}]
\ge{-\tfrac{1}{2}}p_{t-1}^{-1/2}E[p_t-p_{t-1}\mid p_{t-1}]\\
{}&+\tfrac{1}{8}p_{t-1}^{-3/2}E [(p_t-p_{t-1})^2\mid p_{t-1} ]\\
{}&-\tfrac{1}{16}p_{t-1}^{-5/2}E [(p_t-p_{t-1})^3\mid p_{t-1} ].
\end{split}
\label{eq:drift}
\end{equation}

We analyze PBIL first, which includes the UMDA. 
We start by showing that, regardless of $p_0$, the expected time to reach $p_t \in P \coloneqq [0,c\rho/{\mu}] \cup [1-c{\rho}/{\mu},1]$ is $O({\mu}/{\rho^2})$. Via Lemma~\ref{lem:moments}, we have
\begin{equation*}
\begin{split}
E[Y_{t-1}{}&-\sqrt{p_t}\mid Y_{t-1}]
\ge\frac{1}{8}p_{t-1}^{-3/2}\frac{\rho^2}{\mu}p_{t-1}(1-p_{t-1})\\
{}&-\frac{1}{16}p_{t-1}^{-5/2}\frac{\rho^3}{\mu^2}p_{t-1}(1-p_{t-1})(1-2p_{t-1})\\
={}&\frac{\rho^2}{16\mu}p_{t-1}^{-1/2}(1-p_{t-1}) \left(2-\frac{\rho}{\mu p_{t-1}}(1-2p_{t-1}) \right)\\
\ge{}&\frac{\rho^2}{16\mu}p_{t-1}^{-1/2}(1-p_{t-1})\left(2-\frac{1}{c}\right),
\end{split}
\end{equation*}
where the last estimate follows from $p_{t-1} \ge c\rho/\mu$ and from the fact that  $0<p_{t-1}\le \tfrac 12$ implies ${0 \le 1-2p_{t-1}<1}$. Since $p_{t-1}\le \tfrac{1}{2}$, we have ${p_{t-1}^{-1/2}(1-p_{t-1})\ge \tfrac{\sqrt 2}{2}}$. Hence $E[Y_{t-1}-\sqrt{p_t}\mid Y_{t-1}]\ge c_1\rho^2/\mu$, where $c_1=\tfrac{\sqrt2}{32}(2-\tfrac{1}{c})$. Using $q_t=\min\{p_t,1-p_t\}$, we have
\begin{equation}
E[Y_{t-1}-Y_t\mid Y_{t-1}]\ge E[Y_{t-1}-\sqrt{p_t}\mid Y_{t-1}]\ge c_1\rho^2/\mu.\label{eq:driftA}
\end{equation}
By artificially modifying the process $(Y_t)$ once it goes below $c\rho/\mu$, e.g., by defining $(\tilde{Y}_t)$ via $\tilde{Y}_t=Y_t$ if $Y_t\ge c\rho/\mu$ and $\tilde{Y}_t=0$ otherwise, we can ensure that we have a drift of $E[\tilde Y_{t-1} - Y_t \mid Y_{t-1} > 0] \ge c_1 \rho^2 / \mu$ until we reach zero. Such an artificial extension of a process beyond the region of interest, to the best of our knowledge, was in the theory of evolutionary algorithms first used in~\cite{DoerrHK11}. With this artificial extension we can now use the Additive Drift Theorem~\cite{HeY01} with target $\tilde{Y}_t = 0$ and $\tilde{Y}_0=\sqrt{\tfrac{1}{2}}$ and obtain that the expected time for the $\tilde{Y}$-process to reach or go below $\sqrt{c \rho / \mu}$, equivalently to the $p_t$ process reaching $P$, is at most $\frac{\tilde{Y}_0}{c_1 \rho^2/\mu}=\frac{16}{2-1/c} \mu/\rho^2=O(\mu/\rho^2)$. 
We note here that for the UMDA, that is, PBIL with $\rho=1$, the $p_t$ process reaching $P$ hits the absorbing states $\{0,1\}$ since $c\rho /\mu = c/\mu < 1/\mu$ and the frequencies are well-behaved. Hence, we have $E[T]=O(\mu)$ for the UMDA.

Now we continue to discuss the neutral frequency's behavior of PBIL once it has reached $P$. 
W.l.o.g. let $p_t \le c\rho/\mu$. Then the probability that all of the next $\mu\lceil 1/\rho \rceil$ samplings have a zero in the neutral bit position is at least
\begin{align*}
(1-p_t){}&^{\mu\lceil 1/\rho \rceil} \ge \left(1-\frac{c\rho}{\mu}\right)^{\mu\lceil 1/\rho \rceil} \\
\ge{}& \left(1-\frac{c\rho}{\mu}\right)^{\mu\frac{2}{\rho}} = \left(1-\frac{c\rho}{\mu}\right)^{2c\left(\frac{\mu}{c\rho}-1\right)}\left(1-\frac{c\rho}{\mu}\right)^{2c} \\
\ge{}& \exp(-2c)\left(1-2c\frac{c\rho}{\mu}\right) \ge \exp(-2c)(1-2c^2) > 0,
\end{align*}
where the second inequality uses $\lceil 1/\rho \rceil \le 2/\rho$ since $\rho \le 1$, the antepenultimate inequality uses the Bernoulli's inequality, the penultimate inequality uses $\mu \ge 1$ and $\rho \le 1$, and the last inequality uses $c < 1/\sqrt 2$. In this case, the frequency after these $\lceil 1/\rho \rceil$ iterations is
\begin{align*}
p_{t+\lceil 1/\rho \rceil}=(1-\rho)^{\lceil 1/\rho \rceil}p_t \le (1-\rho)^{1/\rho}p_t\le \frac{p_t}{e} \le \frac{c}{e}\frac{\rho}{\mu}. 
\end{align*}
Therefore, with a similar calculation, it is easy to see that the probability that all of the next $\mu\lceil 1/\rho \rceil$ samplings have a zero in the neutral bit position (from the $(t+\lceil 1/\rho \rceil+1)$-th iteration to the $(t+2\lceil 1/\rho \rceil)$-th iteration) is at least $(\exp(-2c)(1-2c^2))^{1/e}$, and ${p_{t+2/\rho}\le (c/e^2)(\rho/\mu)}$. A simple induction gives that the probability that all samplings have a zero in the neutral bit position from the $(t+{(n-1)\lceil 1/\rho \rceil}+1)$-th iteration to the $(t+{n\lceil 1/\rho \rceil})$-th iteration is at least $(\exp(-2c)(1-2c^2))^{1/e^{n-1}}$. Therefore, the probability that only zeros are sampled in the neutral bit position is at least
\begin{align*}
\prod_{i=0}^{\infty}(\exp(-2c){}&(1-2c^2))^{1/e^{i}} = (\exp(-2c)(1-2c^2))^{\sum_{i=0}^{\infty} \frac{1}{e^{i}}} \\
={}& (\exp(-2c)(1-2c^2))^{1/(1-e^{-1})}>0,
\end{align*}
where the last inequality uses $\exp(-2c)(1-2c^2) > 0$. 

Let us divide the run of the EDA into phases. The first phase starts with the first iteration, each subsequent phase starts with the iteration following the end of the previous phase. A phase ends when for the first time after reaching in this phase a $p_t$-value in $P$ an unexpected value is sampled in the neutral bit position. That is, when a one is sampled if the first $p_t$-value in $P$ is in $[0,c \frac{\rho}{\mu}]$ or when a zero is sampled when the first $p_t$-value is at least $1 - c \frac{\rho}{\mu}$. By the above, we know the following about these phases. 
We call a phase successful when it never samples the unexpected value, thus it will not end. From the above calculation, we know the success probability is at least $(\exp(-2c)(1-2c^2))^{1/(1-e^{-1})}$, which is a positive constant. 
Consequently, there is an expected constant number of phases, one of which is successful (namely the last). In each phase, successful or not, it takes an expected time of $O(\mu/\rho^2)$ until the frequency of the neutral bit position reaches a value in $P$. In the successful phase, the frequency then runs away. For the unsuccessful phases, we now show that the phase ends after an expected number of additional $O(1/\rho)$ iterations after reaching a frequency value in~$P$. 

Note that this means analyzing a run of the algorithm starting (in iteration $t+1$) with the neutral frequency $p_t$ in $P$, say w.l.o.g.\ in $[0,c \frac{\rho}{\mu}]$, conditional on the event that at some future time a one is sampled in this position. 

Let $U$ be the event that the phase under investigation is unsuccessful. Let $X \in \{1,2, \dots\}$ be minimal such that in iteration $t+X$ a one is sampled in the neutral bit position of a selected individual. Conditional on $U$, the random variable $X$ is well-defined (that is, finite). For $X=s$ to hold, in particular no one can be sampled in the iterations $t+1, \dots, t+(s-1)$, and this implies not sampling a one in iteration $t+(s-1)$ when the current value of the frequency is $p_t (1-\rho)^{s-1}$. Consequently, the expected length (number of iterations) of an unsuccessful phase is 
\begin{align}
  E[X \mid U] &= \sum_{s = 1}^\infty s \Pr[X = s \mid U] = \frac{1}{\Pr[U]} \sum_{s=1}^\infty s \Pr[X = s] \nonumber\\
  &\le \frac{1}{\Pr[U]} \sum_{s=1}^\infty s \mu p_t (1-\rho)^{(s-1)}\label{eq:sum}
\end{align}
using a union bound over the $\mu$ samples in iteration $t+(s-1)$.

To estimate this expectation, we first compute $\Pr[U]$. For any $k \in \N$, we have
\begin{align*}
  \Pr&[U]  \ge \Pr[X \le k] = 1 - \Pr[X > k]\\
  & = 1 - \prod_{i=1}^{k} \Pr[X > i \mid X > i-1]\\
  & = 1 - \prod_{i=0}^{k-1} \left(1 - p_t(1 - \rho)^i\right)^\mu \\
  & \ge 1 - \exp\left({-\mu p_t} \sum_{i=0}^{k-1}(1-\rho)^i\right)\\
  & = 1 - \exp\left({-\mu p_t} \frac{1 - (1-\rho)^{k}}{1 - (1-\rho)}\right)\\
  & \ge 1 - \left(1 - \frac 12 \mu p_t \frac{1 - (1-\rho)^{k}}{\rho}\right) = \mu p_t \frac{1 - (1-\rho)^{k}}{2 \rho}
\end{align*}  
using the well-known estimates $1+x \le \exp(x)$ valid for all $x \in \R$ and $\exp(-x) \le 1 - \frac x2$ valid for all $0 \le x \le 1$. Taking the supremum over all $k \in \N$, we obtain $\Pr[U] \ge \frac{\mu p_t}{2\rho}$. 

To estimate the infinite sum in~\eqref{eq:sum}, we first recall the elementary formula $\sum_{s=1}^\infty s x^s = \frac{x}{(1-x)^2}$ for $0 < x < 1$, which follows from computing $A := \sum_{s=1}^\infty s x^s  = x \sum_{s=1}^\infty (s-1) x^{s-1} + \sum_{s=1}^\infty x^s = xA + \frac{x}{1-x}$ and solving for $A$. From this, we obtain
\begin{align*}
  \sum_{s=1}^\infty s \mu p_t (1-\rho)^{(s-1)} = \mu p_t \frac{1}{\rho^2}
\end{align*}
and finally 
\begin{align*}
  E[X \mid U] & \le \frac{\mu p_t \frac{1}{\rho^2}}{\frac{\mu p_t}{2\rho}} = \frac{2}{\rho}.
\end{align*}
Consequently, an unsuccessful phase in total takes an expected number of $O(\mu / \rho^2) + O(1 / \rho) = O(\mu / \rho^2)$ iterations. 

By Wald's equation, recalling that we have an expected constant number of unsuccessful iterations, we see that the total time until the frequency of the neutral bit position runs away is $O(\mu / \rho^2)$ iterations. 

For the cGA, in a similar manner as in the first part of the analysis for PBIL, by Lemma~\ref{lem:moments}, equation~(\ref{eq:drift}) becomes
\begin{equation*}
\begin{split}
E[Y_{t-1}-{}&\sqrt{p_t}\mid Y_{t-1}]
\ge\frac{1}{8}p_{t-1}^{-3/2}\frac{2}{K^2}p_{t-1}(1-p_{t-1})\\
={}&\frac{1}{4}p_{t-1}^{-1/2}\frac{1-p_{t-1}}{K^2}
\ge \frac{1}{4}\frac{\sqrt2}{2}\frac{1}{K^2}=\frac{\sqrt 2}{8}\frac{1}{K^2}.
\end{split}
\end{equation*}
Hence,
\begin{equation*}
E[Y_{t-1}-Y_t\mid Y_{t-1}]\ge E[Y_{t-1}-\sqrt{p_t}\mid Y_{t-1}]\ge \tfrac{\sqrt 2}{8}/K^2.
\end{equation*}
Via the Additive Drift Theorem \cite{HeY01} and $Y_0=\sqrt{\tfrac{1}{2}}$, we know that the expected time for the $Y$-process to reach zero is at most $Y_0 / \tfrac{\sqrt 2}{8K^2} = 4K^2$. 
\end{proof}

We now briefly show that the upper bound proof can, under suitable assumptions, also be applied to CE with small modifications. Assume that the learning rate sequence $(\rho_t)$ has both supremum and infimum, and let ${\rho_{\sup}=\sup\{\rho_t \mid t\in \N\}}$ and ${\rho_{\inf}=\inf\{\rho_t \mid t\in \N\}}$. Consider the first generation when the frequency reaches ${\tilde{P} \coloneqq [0,c\rho_{\sup}/{\mu}] \cup [1-c{\rho_{\sup}}/{\mu},1]}$. Following similar arguments as above, we can obtain that the corresponding value in the right side of (\ref{eq:driftA}) becomes $c_1\rho_{\inf}^2/\mu$, and hence the expected reaching time is $O(\mu/\rho_{\inf}^2)$. 

For the neutral frequency's behavior once it has reached $P$, we discuss the case when there exists a positive constant $c'<2$ so that $\rho_{\sup}/\rho_{\inf} \le c'$. In this case, we refine $c\in (1/2, \sqrt{1/(2c')})$. Then  we can obtain that the probability that all samplings have a zero in the neutral bit position from the $(t+i\lceil 1/\rho_{\inf} \rceil+1)$-th iteration to the $(t+(i+1)\lceil 1/\rho_{\inf} \rceil)$-th iteration is at least 
\begin{align*}
\left(\exp\left({-\frac{2c\rho_{\sup}}{\rho_{\inf}}}\right)\left(1-\frac{2c^2\rho_{\sup}^2}{{\rho_{\inf}}}\right)\right)^{1/e^i}\\
\ge \left(\exp({-2cc'})(1-2c^2c')\right)^{1/e^i}>0
\end{align*}
for $i=0,1,\dots$, and the frequency after these $\lceil 1/\rho_{\inf} \rceil$ iterations is at most $c\rho_{\sup}/{(e^{i+1}\mu)}$. Hence, the probability that only zeros are sampled in the neutral bit position is at least
\begin{align*}
\left(\exp({-2cc'})(1-2c^2c')\right)^{1/(1-e^{-1})}>0.
\end{align*}
Similarly, we could calculate that an unsuccessful phase ends after an expected number of additional $O(\rho_{\sup} / \rho_{\inf}^2)$ iterations after reaching a frequency value in~$\tilde{P}$. Hence, for CE, the total time until the frequency of the neutral bit position runs away is $O(\mu / \rho_{\inf}^2)$ iterations.

We note that Corollary~\ref{cor:lower} and Theorem~\ref{thm:upper} give sharp bounds for several hitting times. For the UMDA without margins, the expected first time when the frequency of the neutral bit position is absorbed in 0 or 1 is $\Theta(\mu)$, and the corresponding hitting time is $\Theta(K^2)$ for the cGA. For PBIL without margins and any $c\in (1/2, 1/\sqrt 2)$, the expected first time that the frequency of the neutral bit position hits $c\rho/\mu$ or $1-c\rho/\mu$ is $\Theta({\mu}/{\rho^2})$. As discussed in the second paragraph in Section~\ref{sec:IntroEDA}, these results also hold for the hitting time of the margins $\{1/D, 1-1/D\}$ when running EDAs with such margins.

\section{Extending the Lower Bounds to Bit Positions with Preference: Domination Results}
\label{sec:lowerP}

In the previous Sections~\ref{sec:lower} and~\ref{sec:upper}, we discussed how fast neutral bit positions approach the boundaries of the frequency range. In many situations, e.g., for the benchmark functions \onemax or \leadingones, bit positions are not neutral, but are neutral or have a preference of one bit-value (here the value one). Precisely, we say some bit position, w.l.o.g., the first bit position, of the fitness function $f$ is neutral or prefers a one (we also say \emph{weakly prefers a one}) if and only if 
\[f(0,X_2,\dots,X_D) \le f(1,X_2,\dots,X_D)\] 
for all $X_2,\dots,X_D\in\{0,1\}$. We say that the bit position \emph{weakly prefers a zero} if $f(0,X_2,\dots,X_D) \ge f(1,X_2,\dots,X_D)$ for all $X_2,\dots,X_D\in\{0,1\}$. 

If seems natural that for a bit that weakly prefers a one, the time for its frequency to reach or go below a certain value satisfies the same lower bounds as proven for neutral bit positions, and an analogous statement should be true for bits that weakly prefer a zero. This is what we show in this section. 
%
%

To prove this result, we first establish the following dominance result, which we expect to be useful also beyond this work. It in particular shows that when comparing two runs of an EDA, the first one starting with a higher frequency in a neutral bit position than the second, then in the next generation the frequency in the first run stochastically dominates the one in the second run. This statement remains true if the position in the first run is not neutral, but weakly prefers ones. A simple induction extends this statement to all generations. While not important for our work, we add that we believe that the lemma below does not remain true when both functions can be such that the first bit weakly prefers a one, since other bits' contributions to the fitness should be considered as well. 
Also, simple examples show that our claim is false for the cGA without well-behaved frequencies.

\begin{lemma}
Consider using an $n$-Bernoulli-$\lambda$-EDA to optimize (i) some function~$f$ such that the first bit weakly prefers a one and (ii) some function~$g$ with the first bit being neutral. Assume that the first process is started with a frequency vector~$u^0$ and the second with a frequency vector~$v^0$ such that $u^0_i = v^0_i$ for $i=2, \dots, D$, and $u^0_1 \ge v^0_1$. Assume that in the case of the cGA, the well-behaved frequency assumption holds. 

Let $u^t$ and $v^t$ be the corresponding frequency vectors generated in the $t$-th generation. Then $u^t_1 \succeq v^t_1$ for all $t \in \N$.

Analogously, if $f$ is such that the first bit weakly prefers a zero and we start with $u^0_1 \le v^0_1$, then $u^t_1 \preceq v^t_1$ for all $t \in \N$.
\label{lem:domi}
\end{lemma}

\begin{proof}
We only show the result for weak preference of a one as the other statement can be shown in an analogous fashion or by regarding $({-f}, {-g}, 1-u, 1-v)$ instead of $(f,g,u,v)$. 

We first show the claim for the first iteration and later argue that an easy induction shows it for any time $t$.

For PBIL (or CE), we recall from Section~\ref{sec:Notation} that in the second process in an iteration $t$ started with a frequency $v^{t-1}_1$ of the neutral first bit of $g$, the next frequency of this neutral position is distributed as 
\begin{equation}
(1-\rho) v^{t-1}_1 + \rho \frac 1 \mu Y,\label{eq:pbilupdateneutral}
\end{equation}
where $Y \sim \Bin(\mu,v^{t-1}_1)$. In the first process, a closer inspection of the update rule~\eqref{eq:pbilupdate} shows the frequency of the position weakly preferring a one changes from $u^{t-1}_1$ to 
\begin{equation}
u_1^{t} \sim (1-\rho) u^{t-1}_1 + \rho \frac 1 \mu X,\label{eq:pbilupdateprefone}
\end{equation}
where $X \succeq \Bin(\mu,v^{t-1}_1)$. 

If $u^0_1 \ge v^0_1$, then $\Bin(\mu,u^0_1)$ stochastically dominates $\Bin(\mu,v^0_1)$, and hence $u^1_1 \succeq v^1_1$ by~\eqref{eq:pbilupdateneutral} and~\eqref{eq:pbilupdateprefone}.

For the cGA with the well-behaved frequency assumption, we note that $u^0_1 \ge v^0_1$ implies $u^0_1 = v^0_1$ or $u^0_1 \ge v^0_1 + 1/K$. We only regard the latter, more interesting case. We show $u^1_1 \succeq v^1_1$ using the definition  of domination, that is, that ${\Pr[u^1_1 \le \lambda] \le \Pr[v^1_1 \le \lambda]}$ holds for all $\lambda \in \R$. We discuss differently the following three cases.
\begin{itemize}
\item Assume $\lambda <v^0_1$. Since $u^0_1 - 1/K \ge v^0_1 > \lambda$ from our assumption, we have ${\Pr[u^1_1 \le \lambda] = 0 \le \Pr[v^1_1 \le \lambda]}$.
\item Assume $v^0_1 \le \lambda < u^0_1$. In this case, $\Pr[u^1_1 \le \lambda] \le u^0_1 (1-u^0_1) \le \tfrac 14$ and $\Pr[v^1_1 \le \lambda] = 1 - v^0_1 (1-v^0_1) \ge 1 - \tfrac 14$, which gives the claim.
\item Assume $\lambda \ge u^0_1$. Since $v^0_1+1/K \le u^0_1 \le \lambda$ from our assumption, we have $\Pr[v^1_1 \le \lambda] = 1 \ge \Pr[u^1_1 \le \lambda]$.
\end{itemize}
Hence, we have $u^1_1 \succeq v^1_1$.

To extend our proof to arbitrary generation $t$, we note that if we have $u_1^{t-1} \succeq v_1^{t-1}$, then (see, e.g., \cite[Theorem~12]{Doerr19tcs}) we can find a coupling of the two probability spaces describing the states of the two algorithms at the start of iteration $t$ in such a way that for any point $\omega$ in the coupling probability space we have $u_1^{t-1} \ge v_1^{t-1}$. Conditional on this $\omega$, we can use the above argument for one iteration and obtain $u_1^{t} \succeq v_1^{t}$. This implies that we also have $u_1^{t} \succeq v_1^{t}$ without conditioning on an $\omega$.
\end{proof}

From Lemma~\ref{lem:domi}, we now easily derive that our lower bounds shown in Section~\ref{sec:lower}, suitably adjusted, also hold for bits that weakly prefer one value. Theorem~\ref{thm:lowerP1} discusses the case when a bit weakly prefers a one. 
%
%
%

\begin{theorem}\label{thm:lowerP1}
  Consider using an $n$-Bernoulli-$\lambda$-EDA to optimize some function $f$ with a bit weakly preferring a one. Let $p_t, t = 0, 1, 2, \dots$ denote the frequency of this position after iteration $t$. Let $T_0 = \min\{t \mid p_t \le \tfrac 14\}$ denote the first time this frequency is in $[0,\frac 14]$. 
  
  \begin{enumerate}
  \item Let the EDA be PBIL with learning rate $\rho$ and selection size $\mu$. Then  $E[T_0]=\Omega(\tfrac{\mu}{\rho^2})$, in particular, $E [T_0]=\Omega({\mu})$ for the UMDA and $E [T_0 ]=\Omega(\tfrac{1}{\rho^2})$ for the $\lambda$-MMAS. Again for PBIL, for all $\gamma > 0$ and $T \in \N$ we have
  \[\Pr[\forall t \in [0..T] : p_t > \tfrac 12 - \gamma] \ge 1 - 2 \exp\left({-\frac{\gamma^2 \mu}{2\rho^2 T}}\right).\]
  \item Let the EDA be the cGA with hypothetical population size $K$. Then $E[T_0]=\Omega(K^2)$ and for all $\gamma > 0$ and $T \in \N$ we have
  \[\Pr[\forall t \in [0..T] : p_t > \tfrac 12 - \gamma] \ge 1 - 2 \exp\left({-\frac{\gamma^2 K^2}{2 T}}\right).\]
  \end{enumerate} 
\end{theorem}

\begin{proof}
  Let $g$ be some function with first bit position truly neutral, let $\tilde p_t, t = 0, 1, 2, \dots$ denote the frequency of this position after iteration $t$, and let $\tilde T_0 = \min\{t \mid \tilde p_t \le \tfrac 14\}$ denote the first time this frequency is in $[0,\frac 14]$. Noting that $\tilde p_0 = p_0 = \tfrac 12$, we apply Lemma~\ref{lem:domi} and observe that $p_t \succeq \tilde p_t$ for all $t$. This together with Theorem~\ref{thm:tail} shows the tail bounds.
  
  From $p_t \succeq \tilde p_t$ for all $t$, we also deduce $T_0 \succeq \tilde T_0 \succeq \min\{t \mid p_t \in [0,\tfrac 14] \cup [\tfrac 34, 1]\} =: T_0'$ and thus $E[T_0] \ge E[T_0']$. By Corollary~\ref{cor:lower}, $T_0'$ satisfies the lower bounds we claim for the expectation of $T_0$, and so does $T_0$ itself.
\end{proof}

In an analogous fashion, we obtain Corollary~\ref{cor:lowerP0} the corresponding result for bits weakly preferring a zero.

\begin{corollary}\label{cor:lowerP0}
  Consider using an $n$-Bernoulli-$\lambda$-EDA to optimize some function $f$ with a bit weakly preferring a zero. Let $p_t, t = 0, 1, 2, \dots$ denote the frequency of this position after iteration $t$. Let $T_0 = \min\{t \mid p_t \ge \tfrac 34\}$ denote the first time this frequency is in $[\frac 34, 1]$. 
  
  \begin{enumerate}
  \item Let the EDA be PBIL with learning rate $\rho$ and selection size $\mu$. Then  $E[T_0]=\Omega(\tfrac{\mu}{\rho^2})$, in particular, $E [T_0]=\Omega({\mu})$ for the UMDA and $E [T_0 ]=\Omega(\tfrac{1}{\rho^2})$ for the $\lambda$-MMAS. Again for PBIL, for all $\gamma > 0$ and $T \in \N$ we have
  \[\Pr[\forall t \in [0..T] : p_t < \tfrac 12 + \gamma] \ge 1 - 2 \exp\left({-\frac{\gamma^2 \mu}{2\rho^2 T}}\right).\]
  \item Let the EDA be the cGA with hypothetical population size $K$. Then $E[T_0]=\Omega(K^2)$ and for all $\gamma > 0$ and $T \in \N$ we have
  \[\Pr[\forall t \in [0..T] : p_t < \tfrac 12 + \gamma] \ge 1 - 2 \exp\left({-\frac{\gamma^2 K^2}{2 T}}\right).\]
  \end{enumerate} 
\end{corollary}

We have just extended our previous lower bounds to the case of bit positions preferring a particular value. One may ask whether similar results can be obtained for upper bounds as well. Let us comment on this question. Let us, as in Theorem~\ref{thm:lowerP1} and its proof, denote by $p_t$ the frequencies of a position preferring a one and by $T_0$ the first time this frequency has reached or exceeded a particular value (e.g., $\tfrac 34$ or the upper boundary of the frequency range). Let us denote by $\tilde p_t$ and $\tilde T_0$ the corresponding random variables for a neutral bit position. Then again $p_t \succeq \tilde p_t$ implies $T_0 \preceq \tilde T_0$, so (informally speaking or made precise via a coupling argument) $p_t$ reaches the target not later than $\tilde p_t$. 

However, we do not have any good upper bounds on $\tilde T_0$, neither on its expectation nor in the domination sense. On the technical side, the reason is that we regarded the symmetric process $q_t = \min\{p_t, 1-p_t\}$ in Section~\ref{sec:upper}. The true reason is that also the process itself (when regarding a neutral bit position) is symmetric: With probability $\tfrac 12$ each, the first visit to a boundary is to $\tfrac 1D$ and to $1 - \frac 1D$. However, if the first visit is to $\tfrac 1D$, then it takes quite some time to reach $1 - \tfrac 1D$. Consequently, the distribution of the first hitting time of $1 - \tfrac 1D$ is not well concentrated, and consequently, its expectation might be significantly larger than the first hitting time of $\{\tfrac 1D, 1 - \tfrac 1D\}$. For this reason, we currently do not see how our domination arguments allow to deduce from our results on neutral bit positions reasonable upper bounds on hitting times of frequencies of positions with weak preferences. However, we expect that in most situations where bits with weak preferences occur, one would rather try to exploit the preference to show stronger upper bounds than in the neutral case. For this reason, trying to retrieve information from the neutral case might not be too interesting anyway.

\section{Discussion}
\label{sec:disc}

Just like classic evolutionary algorithms, EDAs are subject to genetic drift and this can, even when using margins for the frequency range, lead to a suboptimal performance. 

For several classical EDAs, this paper proved the first sharp estimates of the expected time the sampling frequency of a neutral bit position takes to leave the middle range $[\frac 13, \frac 34]$ or to reach the boundaries. These times, roughly speaking, are $\Theta(K^2)$ iterations for the cGA and $\Theta(\mu/\rho^2)$ iterations for PBIL (and consequently $\Theta(\mu)$ for its special case UMDA). 

These results are useful both to interpret existing performance results and to set the parameters right in future applications of EDAs. As an example of the former, we note that the recent work~\cite{LehreN19foga} shows that the UMDA with $c \log D \le \mu = o(D)$, $c$ a sufficiently large constant, with $\lambda \le 71 \mu$, and with the margins $1/D$ and $1 - 1/D$, has a weak performance of $\exp(\Omega(\mu))$ on the $D$-dimensional \textsc{DeceptiveLeadingBlocks} benchmark function. This runtime is at least some unspecified, but most likely large polynomial in $D$; it is super-polynomial as soon as $\mu$ is chosen super-logarithmic. For our work, we know that the expected time for the frequency of a neutral bit position to reach the boundaries is only $O(\mu)$ iterations. Since the \textsc{DeceptiveLeadingBlocks} function, similar to the classic \leadingones function, has many bit positions that for a long time behave like neutral, a value of $\mu = o(D)$ results in that a constant fraction of these currently neutral bit positions will have reached the boundaries at least once within the first $D$ iterations. Hence also without looking at the proof of the result in~\cite{LehreN19foga}, which indeed exploits the fact that frequencies reach the margins to show the weak performance, our results already indicate that the weak performance might be caused by the use of parameter values leading to strong genetic drift. 

For a practical use of EDAs, our tail bounds of Theorem~\ref{thm:tail} can be helpful. As a quick example, assume one wants to optimize some function via the cGA and one is willing to spend a computational budget of $F$ fitness evaluations. Since the cGA performs two fitness evaluations per iteration, this is equivalent to saying that we have a budget of $T = F/2$ iterations. From Theorem~\ref{thm:tail}(b), with $\gamma = 1/4$, and a simple union bound over the $D$ bit positions, we see that the probability that one of the (temporarily) neutral bit positions leaves the middle range $[\tfrac 14, \frac 34]$ is at most ${D \cdot 2 \exp({-\frac{\gamma^2 K^2}{2T}})}$. Consequently, by using a parameter value of $K \ge \frac 1\gamma \sqrt{F \ln(20D)}$, we obtain that with probability at least 90\% no neutral position leaves the middle range (and, with the results of Section~\ref{sec:lowerP}, no position that weakly prefers one bit value leaves the middle range into the opposite direction). Phrased differently, this means that within this time frame, only those positions approach the boundaries for which there is a sufficiently strong signal from the objective function. While this consideration cannot determine optimal parameters for each EDA and each objective function, it can at least prevent the user from taking parameters that are likely to give an inferior performance due to genetic drift. Since genetic drift has been shown to lead to a poor performance in the past, we strongly recommend to choose the parameters $K$ and $\mu$ large enough so that estimates based on Theorem~\ref{thm:tail} guarantee that positions without a fitness signal stay in the middle range. 

\section*{Acknowledgments}

This work was supported by a public grant as part of the Investissement d'avenir project, reference ANR-11-LABX-0056-LMH, LabEx LMH, in a joint call with Gaspard Monge Program for optimization, operations research and their interactions with data sciences.

This work was also supported by the Program for Guangdong Introducing Innovative and Enterpreneurial Teams (Grant No. 2017ZT07X386); Guangdong Basic and Applied Basic Research Foundation (Grant No. 2019A1515110177); Shenzhen Peacock Plan (Grant No. KQTD2016112514355531); the Program for University Key Laboratory of Guangdong Province (Grant No. 2017KSYS008); and the Science and Technology Innovation Committee Foundation of Shenzhen (Grant No. JCYJ20190809121403553). We thank Guangwen Yang, Haohuan Fu, and Xin Yao for their valuable support of this work.



%

\end{document}